\documentclass[sigconf,nonacm]{aamas}

\acmConference[arXiv preprint]{arXiv preprint}{}{}{}

\usepackage{balance}

\usepackage[utf8]{inputenc}
\usepackage{booktabs}
\usepackage{tikz}
\usetikzlibrary{arrows.meta, positioning, shapes.geometric, calc}
\usepackage{algorithm}
\usepackage{algorithmic}

\newtheorem{definition}{Definition}
\newtheorem{theorem}{Theorem}
\newtheorem{lemma}[theorem]{Lemma}
\newtheorem{proposition}[theorem]{Proposition}
\newtheorem{example}{Example}
\newcommand{\lo}{\ell}
\newcommand{\hi}{u}
\newcommand{\vacant}{v}
\newcommand{\nothaving}{\neg\mathsf{holds}}
\newcommand{\belief}{P}

\title[Inner-Learning for Trick-Taking Games]{Inner-Learning: From Hand-Coded Beliefs to\\
Declarative Knowledge in Trick-Taking Games}

\author{Stefan Edelkamp}
\affiliation{\institution{Department of Theoretical Computer Science and
Mathematical Logic, Faculty of Mathematics and Physics Charles~University,
Prague}
\country{Czech Republic}}
\email{edelkamp@ktiml.mff.cuni.cz}

\begin{document}

\begin{abstract}
We present a declarative approach to belief-state modeling in trick-taking card
games. Starting from a state-of-the-art Skat player, we automatically extract
its hand-coded inference rules into a domain-specific language built on the
primitive $\neg\mathsf{holds}(P,C)$, denoting that player $P$ cannot hold any card in
set $C$. Round-trip code generation verifies that the extracted rules
reproduce the original engine's intended behavior, up to three audited
bugs in the original that the declarative engine corrects. We then grow the rule base through symmetric derivation and
annotation of an expert strategy document---unifying code, symmetry,
and expertise in one representation. We extend the knowledge representation with \emph{cardinality constraints}:
at-most-$k$ and at-least-$k$ bounds on the cards a player holds from a
named group, expressed in the same declarative language and enforced
at world-generation time.
Confidence parameters connect symbolic rules to probabilistic belief
updates. Self-play evaluation shows the declarative engine measurably
improving winning percentage and score over the original hand-coded
engine, with the extracted rules, the human-annotated rules, and the
cardinality constraints each contributing to the effect.

\end{abstract}


\maketitle

\section{Introduction}

Card games with imperfect information require agents to maintain a
\emph{belief state}---a constraint set or probability distribution over
possible world states consistent with observed play. In trick-taking games
such as Skat, Bridge, and Belote, each trick reveals information: a player
who fails to follow suit cannot hold remaining cards of that suit; a player
who follows with a specific card reveals gaps in their holding.
Maintaining accurate belief states is central to strong play in
imperfect-information games, yet the inference logic is rarely
extracted or formalized independently of the game engine.

Human experts encode these inferences as production rules; game engines implement
them as chains of bitmask operations interleaved with perspective-dependent
routing---correct but difficult to audit, extend, or port.

We contribute a declarative domain-specific language for belief-state
inference in trick-taking card games. The language centers on a single
primitive, $\neg\mathsf{holds}(P,C)$, plus a vocabulary of
game-agnostic predicates. A code generator compiles these rules into
efficient belief-update functions, handling perspective routing,
filter modes, and phase ordering automatically. Rules are extracted from a production
Skat engine, verified for equivalence, and extended through symmetric
derivation and annotation of an expert strategy document---demonstrating
the DSL's expressiveness across knowledge sources (code extraction,
symmetric derivation, expert annotation). We introduce \emph{cardinality constraints} that bound the number of
cards a player holds from a named group (e.g.~Jacks, trump suit,
black Jacks), expressed in the same DSL and enforced via a world
filter at generation time. The constraint hypergraph is nearly
acyclic, making a single forward pass of bounds consistency
sufficient. Custom group masks let
users define domain-specific groups without engine changes.
Finally,
we provide a debugging pipeline with instrumented code generation and
trick-by-trick replay that enables per-rule auditing of inference
correctness. We call this methodology \emph{Inner-Learning}---extracting,
formalizing, and re-injecting the inner inference logic of a game engine
as declarative knowledge (Figure~\ref{fig:framework}).

The core contribution is, to our knowledge, the first belief-state
inference framework with provable monotone refinement and one-sweep
consistency guarantees that is distilled directly from production
game code---a concise yet flexible knowledge representation for
belief-state inference in trick-taking card games:
a single primitive, $\neg\mathsf{holds}(P,C)$, whose $H \times P$
bitmask matrix admits an exact dual, the possibility matrix, extended
by cardinality constraints over named card groups whose nearly
acyclic structure admits one-sweep bounds consistency
(Theorem~\ref{thm:sweep}). On this representation, extraction,
symmetric derivation, and expert annotation build a 139-rule base
that reproduces a production Skat engine's intended behavior on 965
of 1{,}000 games (the three audited divergences are bugs in the
original), and the declarative engine reaches 85.84\% wins against
the original's 84.65\% in self-play, with each data-backed layer
contributing a statistically significant gain
(Section~\ref{sec:experiments}).

The paper is structured as follows. Section~2 positions the approach
against related work in imperfect-information games, knowledge
representation, and specification mining. Section~3 fixes the formal
model---deck, tricks, observations, and the $\neg\mathsf{holds}$
primitive with its dual possibility-matrix representation.
Section~4 develops the knowledge representation, including Bayesian
belief tracking and the cardinality-constraint fragment with its
one-sweep consistency result. Section~5 presents the rule language,
the code generator, and the debugging pipeline. Section~6 reports the
experimental evaluation, Section~7 covers reproducibility,
deployment, and limitations, and Section~8 concludes.

\begin{figure}[t]
\centering
\resizebox{\columnwidth}{!}{
\begin{tikzpicture}[
  node distance=5mm and 2.5mm,
  src/.style={rectangle, draw, rounded corners, fill=gray!10, minimum width=19mm, minimum height=6mm, align=center, font=\scriptsize},
  proc/.style={rectangle, draw, rounded corners, fill=blue!12, minimum width=19mm, minimum height=6mm, align=center, font=\scriptsize},
  outbox/.style={rectangle, draw, rounded corners, fill=green!10, minimum width=18mm, minimum height=6mm, align=center, font=\scriptsize},
  arrow/.style={-{Stealth}, thick},
  dasharrow/.style={-{Stealth}, dashed, thick, gray}
]
\node[src] (code) at (0,0) {Production Engine\\(C\texttt{++} source)};
\node[src, right=of code] (sym) {Symmetric\\Derivation};
\node[src, right=of sym] (doc) {Expert Strategy\\Documents};

\node[proc, below=of code] (extract) {Automated\\Extraction};
\node[proc, below=of sym] (llm) {LLM-Assisted\\Annotation};

\node[cylinder, shape border rotate=90, draw, thick, fill=yellow!20, minimum width=24mm, minimum height=10mm, align=center, font=\small, below=of llm, yshift=-2mm] (dsl) {Declarative\\Rule Base};

\node[outbox, below left=of dsl, xshift=-10mm, yshift=-2mm] (belief) {Belief-State\\Update};
\node[outbox, below=of dsl, yshift=-2mm] (compile) {Code Generation\\\& Deployment};
\node[outbox, below right=of dsl, xshift=10mm, yshift=-2mm] (audit) {Explainable\\Auditing};
\node[outbox, below=of compile, yshift=-4mm] (world) {World\\Generation};
\node[outbox, left=of world, xshift=-2mm] (opp) {Opponent\\Modeling};
\node[proc, above=of audit, yshift=2mm, fill=gray!15] (replay) {Rule-Annotated\\Replay};

\draw[arrow] (code) -- (extract);
\draw[arrow] (sym) -- (llm);
\draw[arrow] (doc) -- (llm);

\draw[arrow] (extract) -- (dsl);
\draw[arrow] (llm) -- (dsl);

\draw[arrow] (dsl) -- (belief);
\draw[arrow] (dsl) -- (compile);
\draw[arrow] (dsl) -- (audit);
\draw[arrow] (compile) -- (world);
\draw[arrow] (compile) -- (opp);
\draw[arrow] (audit) -- (replay);

\draw[arrow, bend left=25] (compile.west) to (code.west);

\end{tikzpicture}
}
\caption{The Inner-Learning framework. Rules are extracted from a production
engine, extended by symmetry, and annotated from expert documents.
The declarative rule base supports belief-state updates, code generation,
explainable auditing, rule-annotated replay, world generation, and opponent
modeling.}
\label{fig:framework}
\end{figure}
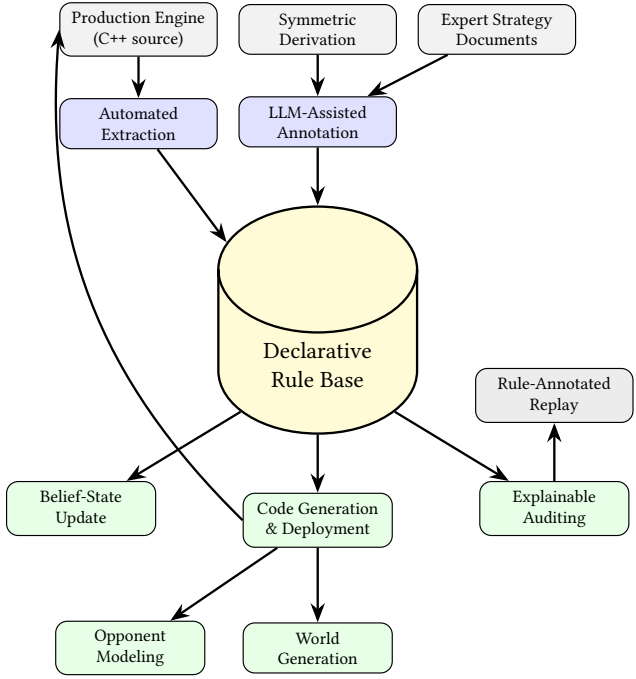

\section{Related Work}

Landmark results in imperfect-information games include DeepStack
\cite{moravcik2017deepstack} combined neural networks with depth-limited
search for heads-up no-limit hold'em. Libratus \cite{brown2018libratus}
introduced safe subgame solving via nested counterfactual regret minimization.
Pluribus \cite{brown2019pluribus} extended these ideas to multiplayer poker.
These methods focus on computing equilibrium strategies rather than
maintaining explicit belief states.

In trick-taking card games, Bridge has received the most attention.
Frank and Basin \cite{frank98} pioneered search with incomplete information,
and Cazenave and Ventos \cite{cazenave19} developed the $\alpha\mu$
algorithm. For Skat, the Kermit system by Buro et al.~\cite{buro2009kermit}
established the first expert-level computer player using PIMC search
and human-data-driven state evaluation, later extended with deep neural
sampling~\cite{solinas2019skat} and policy-based inference~\cite{rebstock2019policy}.
Kupferschmid and Helmert~\cite{kupferschmid2006skat} contributed
a double-dummy Skat solver achieving 10~ms per hand. Our work is
complementary: rather than improving search or learning, we extract and
formalize the belief-update logic as declarative rules.

DeepNash
\cite{perolat2022deepnash} mastered Stratego using model-free multiagent RL with Regularized Nash Dynamics
(RNaD). Kub{\'{\i}}cek et al.\ \cite{kubicek2024sepot} extended RNaD
with lookahead search (SePoT), outperforming the original on Goofspiel
and other imperfect-information benchmarks in OpenSpiel.
These methods compute equilibrium strategies rather than
maintaining explicit belief states, but the underlying abstraction
techniques are complementary to our knowledge representation.

Set-based retrograde analysis \cite{stone2026set} and the Belief Forest
framework \cite{zhang2026belief} provide formal connections to our work.
Dechter's text on constraint
processing~\cite{dechter2003constraint} provides the algorithmic
foundation for bounds consistency, and R{\'e}gin's flow-based
propagator~\cite{regin1996gcc} achieves arc consistency for global
cardinality constraints. In cooperative card games, the Bayesian Action Decoder for Hanabi
\cite{foerster2019bad} maintains factorized belief spaces with iterative
self-consistent refinement---conceptually similar to our IPF step for
resolving overlapping cardinality constraints. In expert systems, MYCIN
\cite{buchanan1984mycin} established the paradigm of rules with confidence
parameters. Neurosymbolic approaches \cite{garcez2020neurosymbolic} and
learned embeddings for game abstraction \cite{liu2026embedding} suggest
that declarative rules could provide informative priors for neural belief
update models.

Our extraction step connects to specification mining: dynamic invariant
detection~\cite{ernst2007daikon} and mining specifications from execution
traces~\cite{ammons2002mining} infer likely program properties from
observed behavior, which is what the code extraction does for the
engine's hand-coded inference logic---with the difference that our rules
carry game semantics, phases, and confidence parameters rather than
program-state predicates.

\section{Formal Model}

We model belief states over card holdings using a relational framework
that unifies the extracted rules under a single primitive. The following
definitions establish the game ontology shared by all trick-taking games.

\begin{definition}[Card deck and players]
A card deck $\mathcal{D}$ is a finite set of $D$ cards dealt to $P$ players.
Each card $c \in \mathcal{D}$ has a \emph{suit} $\sigma(c) \in \Sigma$ (where
$|\Sigma| = S$) and a \emph{rank} $\rho(c) \in \mathcal{R}$ (where $|\mathcal{R}| = R$).
A subset of ranks may be designated as \emph{trump}, with a total order
$\prec_T$ on trump cards. Non-trump suits have a total order $\prec_s$
within each suit $s$.
\end{definition}

For Skat: $P{=}3$, $D{=}32$, $S{=}4$, $R{=}8$, with ranks
7--A plus the Jacks.
The trump order runs from $\clubsuit J$ (highest) over
$\spadesuit J$, $\heartsuit J$, $\diamondsuit J$, A, 10, K, Q,
9, 8, to 7 (lowest).

\begin{definition}[Trick]
A \emph{trick} $\tau$ of size $k$ is an ordered tuple
$(c_1, c_2, \ldots, c_k)$ where $c_1$ is the \emph{lead} card and
$0 \le k \le P$ is the number of cards played so far (one per player).
The lead establishes the \emph{suit to follow}: each subsequent card $c_i$
($i > 1$) must belong to $\sigma(c_1)$ unless the player holds no such
card, in which case any card (including trump) may be played.
\end{definition}

\begin{definition}[Player roles and observations]
Roles are a partition of $\{0,\ldots,P-1\}$ into
named positions. Let $H$ be the number of \emph{holders}---entities that
can possess cards: players plus, in some games, a talon or dog.
Each holder $h \in \{0,\ldots,H-1\}$ may be a player or a special
non-playing position whose cards are initially hidden but may be
revealed during play. For Skat: $H = 4$ holders ($\mathsf{D}$ declarer,
$\mathsf{O1}$ first opponent, $\mathsf{O2}$ second opponent, and the
skat). Each player $v \in \{0,\ldots,P-1\}$ maintains a \emph{view}
of the game state---their own hand plus public information (played
cards, scores, bidding).
\end{definition}

\begin{definition}[Observation types]
When a card is played, each observer classifies it relative to the
lead of the current trick:
\begin{itemize}
\item $\mathsf{dropSuit}(c_i, c_1)$: $c_i$ does not follow the suit of $c_1$
  and $c_1$ is not trump. The player is \emph{void} in $\sigma(c_1)$.
\item $\mathsf{dropTrump}(c_i, c_1)$: $c_i$ does not follow suit and $c_1$ is
  trump. The player is \emph{void} in trump.
\item $\mathsf{obeySuit}(c_i, c_1)$: $c_i$ follows the suit of $c_1$.
  The player \emph{possesses} at least one card of $\sigma(c_1)$.
\item $\mathsf{obeyTrump}(c_i, c_1)$: $c_i$ is also a trump.
\end{itemize}
\end{definition}

\section{Knowledge Representation}
\label{sec:knowledge}

Each observation type constrains the possible distribution of remaining
cards. Hard constraints (certain knowledge) are captured by bitmasks;
soft constraints (probabilistic inference) use continuous belief values
maintained alongside the hard state. Together they drive world sampling
for the game solver.

\begin{definition}[Knowledge primitive]
The knowledge base tracks a single relation:
$\neg\mathsf{holds}(h, C)_v$ meaning ``from viewer $v$'s perspective,
holder $h$ does not contain any card in set $C$,''
where $h \in \{0,\ldots,H-1\}$ and
$v \in \{0,\ldots,P-1\}$ is the observing player. $C \subseteq \mathcal{D}$ is
represented as a bitmask over $\{0,\ldots,D-1\}$.
\end{definition}

The full knowledge state is an $H{\times}P$ matrix of bitmasks:
entry $\mathcal{K}[h][v]$ holds the cards $c$ excluded for holder
$h$ from viewer $v$'s perspective, that is, those with
$\neg\mathsf{holds}(h,\{c\})_v$.

\noindent In Skat, the 32-card deck is encoded as a bitmask over 32~bits.
Named sets include the suit masks $\mathsf{mask}[s]$, the trump
mask $\mathsf{trumpmask}$ (Jacks plus the declared suit), the value
sets (Aces, Tens, Kings, Queens, non-values), and the individual
Jacks. Card sets are combined via bitwise operations.

\begin{definition}[Adding mechanism]
Each inference rule computes a set $\mathsf{adding} \subseteq \mathcal{D}$
and a holder $h$. The update is:
\[
\mathcal{K}[h][v] \leftarrow \mathcal{K}[h][v] \cup
(\mathsf{adding} \cap \mathsf{others}[v] \setminus \mathsf{played})
\]
where $\mathsf{others}[v] = \mathcal{D} \setminus
(\mathsf{myHand}[v] \cup \mathsf{played})$ is the set of unseen cards from
$v$'s perspective. Cards already played or in the observer's own hand
cannot be assigned to opponents.
\end{definition}

\begin{definition}[Possibility matrix — the dual representation]
The complement of the knowledge matrix $\mathcal{K}$ yields the
\emph{possibility matrix}
$\mathcal{P}$, a $(D \times H)$ boolean per viewer indicating which
players could still hold each card:
\[
\mathcal{P}_v[c][h] = \begin{cases}
1 & \text{if } c \notin \mathcal{K}[h][v] \text{ (holder $h$ could hold $c$)}\\
0 & \text{if } c \in \mathcal{K}[h][v] \text{ (holder $h$ cannot hold $c$)}
\end{cases}
\]
\end{definition}
The two representations carry identical information---each is the complement
of the other, transposed---and they are duals, not competitors.

\begin{theorem}[Duality]
For any viewer $v$ and card $c$, exactly one of the following holds:
(i) $c \in \mathsf{played}$ (observed), (ii) $c \in \mathsf{myHand}[v]$
(owned), (iii) $\sum_{h=0}^{H-1} \mathcal{P}_v[c][h] \geq 1$ (at least one
possible holder remains). A world is \emph{inconsistent} iff any card has
$\sum_i \mathcal{P}_v[c][h] = 0$. A card is \emph{known} to be with player
$i$ iff $\mathcal{P}_v[c][h] = 1$ and $\mathcal{P}_v[c][j] = 0$ for all
$j \neq i$.
\end{theorem}
\begin{proof}
Partition the deck from viewer $v$'s perspective: played cards are
observed, $v$'s own cards are owned, and every remaining card must
sit with exactly one of the other holders. The three cases are
therefore exhaustive and mutually exclusive, so exactly one holds for
every card. A world is inconsistent precisely when some card has no
possible holder left, which is the case $\sum_h \mathcal{P}_v[c][h] =
0$; a card is known exactly when a unique possible holder remains,
which is condition (iii) specialized to a single $h$.
\end{proof}

\begin{proposition}[Monotone refinement]
\label{prop:monotone}
Rule applications refine beliefs monotonically: the
$\neg\mathsf{holds}$ matrix only grows, the possibility matrix only
shrinks, and every world once excluded stays excluded. Consequently
rule layers form a refinement chain---adding a layer can tighten
inference but can never contradict an earlier exclusion, which
structurally explains the no-harm behavior of the ablations in
Section~\ref{sec:experiments}.
\end{proposition}
\begin{proof}
The adding mechanism writes
$\mathcal{K}[h][v] \leftarrow \mathcal{K}[h][v] \cup
(\mathsf{adding} \cap \mathsf{others}[v] \setminus \mathsf{played})$,
a union with a possibly empty set; the only other write is the purge
of played cards, which are observed and hence excluded for every
holder anyway. Both operations are monotone under set inclusion, so
$\mathcal{K}$ grows and its complement $\mathcal{P}_v$ shrinks
pointwise over time. The pigeonhole interval updates of
Section~\ref{sec:knowledge} are monotone functions of the remaining
possible set, so all group bounds inherit the refinement.
\end{proof}

While $\mathcal{P}$ makes per-card queries efficient, it cannot directly
express cross-card constraints such as ``player $i$ holds exactly 5~trumps.''
The pigeonhole principle bridges this gap: if $k$ cards of a suit are each
possible only for a subset of players whose combined capacity is exactly
$k$, those cards are forced to those players. A special case already
implemented in the engine is transitive closure across two players:

\begin{definition}[Transitive closure]
When two players together are known to lack every remaining card of a
suit, the third player must hold them all:
\[
\begin{array}{c}
|\mathcal{K}[a][v] \cap \mathcal{K}[b][v] \cap \mathsf{suit}(s)|
= |\mathsf{others}[v] \cap \mathsf{suit}(s)| \\[2pt]
\Longrightarrow\; \neg\mathsf{holds}(c, \overline{\mathcal{K}[a][v] \cup
\mathcal{K}[b][v]})_v
\end{array}
\]
where $\{a,b,c\} = \{0,1,2\}$ are the three player positions. This meta-rule operates on the dual, positive representation---what
players are known to hold---computed from the $\neg\mathsf{holds}$ matrix. Extension to
rational-valued $\mathcal{P}_v[c][h] \in [0,1]$ with per-card
normalization would upgrade from hard constraints to Bayesian
distributions.
\end{definition}

Algorithm~\ref{alg:belief} formalizes the per-trick belief update loop
implemented by the engine. For each trick, every participant's action is
classified relative to the lead card, rules of the matching type are
evaluated, and their exclusion sets are accumulated into a bitmask that is
then routed to the appropriate $\neg\mathsf{holds}$ matrix entries for all
three viewer perspectives.

\begin{algorithm}[t]
\caption{Runtime belief update: processing one card.}
\label{alg:belief}
\begin{algorithmic}[1]\footnotesize
\REQUIRE Partial trick $(t_1, \ldots, t_k)$, $1 \le k \le P$, viewer $v$, matrix $\mathcal{K}$, rule base $\mathcal{R}$
\ENSURE Updated $\mathcal{K}[*][v]$
\IF{$k \geq 2$}
  \STATE Classify $t_k$ relative to $t_1$: drop-suit, obey-suit, drop-trump, obey-trump
  \STATE Let $p \gets$ player of $t_k$; $S \gets \varnothing$
  \FORALL{rules $r \in \mathcal{R}$ matching the observation type}
    \IF{condition $\phi_r$ holds for viewer $v$}
      \STATE $S \gets S \cup C_r \cap \mathsf{others}[v] \cap \overline{\textsf{played}}$
    \ENDIF
  \ENDFOR
  \STATE $h \gets p$ if $r$ is responder-targeted, else $h \gets \text{issuer}$
  \STATE Write $S$ to $\mathcal{K}[h][v]$
\ENDIF
\STATE Purge: $\mathcal{K}[*][v] \gets \mathcal{K}[*][v] \setminus \textsf{played}$
\end{algorithmic}
\end{algorithm}

\subsection{Bayesian Belief Tracking}

Hard constraints determine which worlds are impossible. For the remaining
uncertainty, we maintain a continuous belief distribution over an information
set $\Omega_v$ (the worlds consistent with player $v$'s observations). Let
$\omega \in \Omega_v$ denote a complete card assignment. Initially uniform,
$P^{(0)}(\omega) = 1/|\Omega_v|$, the belief is updated after each action
$a$ (a card played by some player) via Bayes' rule:
\[
P^{(t+1)}(\omega) = \frac{P(a \mid \omega) \cdot P^{(t)}(\omega)}
{\sum_{\omega' \in \Omega_v} P(a \mid \omega') \cdot P^{(t)}(\omega')}
\]
The likelihood $P(a \mid \omega)$ is 0 if $\omega$ is inconsistent
with the action and uniform otherwise. Let $P_v[c][h]$ denote the
probability that card $c$ resides with holder $h$ from observer $v$'s
perspective.

\begin{definition}[World Consistency]
A world $\omega$ is \emph{inconsistent} with belief state $P_v$ iff
there exists a card $c$ and player $i$ such that
$c \in \mathsf{hand}_i(\omega)$ but $P_v[c][i] = 0$.
\end{definition}

\begin{definition}[Known Card]
Card $c$ is \emph{known} to be with holder $h$ iff
$P_v[c][i] = 1$ and $P_v[c][j] = 0$ for all $j \neq i$.
\end{definition}

Maintaining the full posterior explicitly is exponential in the
unseen cards, so the engine approximates the marginals directly.
Hard $\neg\mathsf{holds}$ constraints zero out impossible worlds
before sampling, and the remaining probability mass is distributed
proportionally to each player's vacant card spaces. Where cardinality
constraints apply, exchangeability within a group yields the
truncated-hypergeometric marginals of
Equation~\ref{eq:hyper}---conditioned on the current bounds, the
number of group cards held follows a truncated hypergeometric
distribution---and overlapping groups are resolved by the iterative
proportional fitting described there, which converges in 3--4
iterations in practice. The resulting per-card probabilities drive
the world filter and yield the colored belief fills shown in
Figure~\ref{fig:beliefviz}. This marginal machinery, not the full
posterior, is what the engine maintains at run time.

\subsection{Cardinality Constraints}

Boolean encodings of expert knowledge, where each proposition
$x_{ijc}$ asserts that player~$i$ holds card~$c$ from the perspective
of viewer~$j$, offer a clean logical language but require heavy
solvers for propagation. Fortunately, expert
inferences in Skat almost never demand arbitrary clauses. The required
fragment is \emph{cardinality constraints}: at-most-$k$ and
at-least-$k$ bounds on the number of cards from a designated group
that a player holds.

\begin{definition}[Cardinality Constraint]
For viewer $j \in \{0,\ldots,P-1\}$, target holder $h \in \{0,\ldots,H-1\}$
(players plus the skat),
card group $g \subseteq \{0,\ldots,D-1\}$, and bounds
$0 \le \lo_{hj}^g \le \hi_{hj}^g \le D$, a \emph{cardinality constraint}
$\lo_{hj}^g \le |H_h \cap g| \le \hi_{hj}^g$ asserts that the
number of cards from group~$g$ in holder~$h$'s hand lies within
the given interval.  For the viewer's own hand ($h=j$), the bounds
collapse to exact counts determined by the dealt cards.
\end{definition}

\begin{definition}[Group Structure]
A \emph{card group} is a bitmask $g \subseteq \{0,1\}^{D}$. A set
$\mathcal{G}$ of groups is \emph{nearly acyclic} if its members are
pairwise disjoint or nested: for any $g_1, g_2 \in \mathcal{G}$,
either $g_1 \cap g_2 = \emptyset$, $g_1 \subset g_2$, or $g_2
\subset g_1$.
\end{definition}

The default $\mathcal{G}$ contains seven groups: black Jacks
$g_{\textit{BJ}} = \{\clubsuit J, \spadesuit J\}$, red Jacks
$g_{\textit{RJ}} = \{\heartsuit J, \diamondsuit J\}$, all Jacks
$g_{\textit{J}} = g_{\textit{BJ}} \cup g_{\textit{RJ}}$, and the
plain cards of each suit $g_s = \mathsf{suitMask}[s] \setminus
g_{\textit{J}}$.  In Grand games the trump group equals
$g_{\textit{J}}$; in suit games $g_{\textit{trump}} = g_{\textit{J}}
\cup g_s$ for the trump suit~$s$.  Users may add custom groups
(e.g., black Jacks, high trump) without engine changes.  The
nesting $g_{\textit{BJ}} \subset g_{\textit{J}} \subseteq
g_{\textit{trump}}$ makes $\mathcal{G}$ nearly acyclic, so a single
forward sweep per card event suffices for group-bound consistency.

We extend the $K_{hj}$ matrix with lower and upper bound matrices
$\lo_{hj}^g$ and $\hi_{hj}^g$ for these groups. The initial bounds
are seeded from the game contract:
a Grand declarer holds at least one Jack; a suit-game declarer
holds at least four trumps (five in a hand game). All other
recipients start with the structural bounds $\lo=0$, $\hi=\vacant_h$.

Bounds tighten monotonically. After each card play, $\lo_{hj}^g$ is
the maximum of the seeded floor (adjusted for Jacks or trumps already
played by holder~$h$) and a pigeonhole lower bound:
$\vacant_h - |\text{possible}_{hj} \setminus g|$, where
$\text{possible}_{hj}$ is the set of cards not yet excluded for
holder~$h$ by the $K_{hj}$ matrix. $\hi_{hj}^g$ is
$\min(\vacant_h, |\text{possible}_{hj} \cap g|)$.
The per-card exclusion rules thus feed the intervals automatically:
every card ruled out of a group shrinks the possible set and
tightens every group bound that intersects it.

\begin{definition}[Group-Bound Consistency]
\label{def:gbc}
A belief state is \emph{group-bound consistent} if for every holder~$h$,
viewer~$j$, and group~$g$, the interval $[\lo_{hj}^g, \hi_{hj}^g]$
is the tightest possible given the current constraint set: no bound
can be raised (for~$\lo$) or lowered (for~$\hi$) without
contradicting some constraint.
\end{definition}

\begin{theorem}[One-sweep group-bound consistency]
\label{thm:sweep}
For a nearly-acyclic group family $\mathcal{G}$, processing the
groups in inclusion order (leaves before their supersets) and
applying the pigeonhole interval update once per group yields a
group-bound consistent state, provided no $\hi{=}0$ cascade fired.
If a cascade fires, a second sweep over the affected groups restores
consistency.
\end{theorem}
\begin{proof}
A nearly-acyclic family is laminar, so the inclusion relation forms a
forest. For disjoint groups the pigeonhole lower bounds
$\vacant_h - |\text{possible}_{hj} \setminus g|$ count disjoint card
sets and cannot double-count; for nested pairs $g_1 \subset g_2$ the
bound of $g_1$ never weakens that of $g_2$, since
$|\text{possible}_{hj} \cap g_1| \le |\text{possible}_{hj} \cap
g_2|$. Leaf-first processing therefore incorporates every exclusion
exactly once: the sweep is a topological evaluation of the forest,
and each interval is computed from the possible set that already
reflects all per-card rules, so repeated application of the update rules cannot tighten
any bound further (Definition~\ref{def:gbc}). A $\hi{=}0$ cascade writes new
exclusions into the $\neg\mathsf{holds}$ matrix, changing the
possible set; because the forest has no cycles, one additional sweep
propagates those changes completely. Full arc consistency for the
global cardinality constraint would require R{\'e}gin's flow-based
propagator~\cite{regin1996gcc}; the laminar structure is what makes
the cheaper sweep exact in our setting.
\end{proof}

\begin{algorithm}[t]
\caption{One-sweep group-bound update for a card event.}
\label{alg:sweep}
\begin{algorithmic}[1]\footnotesize
\REQUIRE groups $\mathcal{G}$ in inclusion order (leaves first),
matrices $\mathcal{K}$, $\lo$, $\hi$, event card $c$
\STATE update $\mathcal{K}$ with the per-card exclusions triggered by the event
\STATE $cascade \gets \mathsf{false}$
\FORALL{$g \in \mathcal{G}$ in inclusion order}
  \STATE $p \gets |\text{possible} \cap g|$;\quad
         $u \gets |\text{possible} \setminus g|$
  \STATE $\lo^g \gets \max(\lo^g,\, \vacant - u)$;\quad
         $\hi^g \gets \min(\hi^g,\, \vacant,\, p)$
  \IF{$\hi^g = 0$}
    \STATE write all cards of $g$ into the $\neg\mathsf{holds}$ matrix
    \STATE $cascade \gets \mathsf{true}$
  \ENDIF
\ENDFOR
\IF{$cascade$} \STATE repeat the sweep once \ENDIF
\end{algorithmic}
\end{algorithm}

Collapse points are particularly valuable:
\begin{itemize}
\item $\lo = \hi$: the exact number of group cards in the hand
  is known. For the viewer's own cards this holds trivially from
  the first trick; for opponents it typically occurs in midgame.
\item $\hi = 0$: the target's hand contains no card from the
  group, which writes all remaining group cards into the
  $\nothaving$ matrix for that target---a hard exclusion cascade.
\item $\lo = |\text{possible}_{hj} \cap g|$: the target must hold
  every remaining group card, eliminating those cards for all other
  recipients.
\end{itemize}

Conditional-probability compilation follows from exchangeability within
a group. Conditioned on $\lo_{hj}^g \le m \le \hi_{hj}^g$, the number
$m$ of group cards held by holder~$h$ follows a truncated hypergeometric
distribution:
\begin{equation}
\label{eq:hyper}
P(m \mid \lo, \hi) \propto \binom{G}{m} \cdot \binom{U-G}{\vacant_h - m},
\quad \lo \le m \le \hi
\end{equation}

\begin{algorithm}[t]
\caption{Iterative proportional fitting for overlapping group bounds.}
\label{alg:ipf}
\begin{algorithmic}[1]\footnotesize
\REQUIRE marginals $P_v[c][h]$, group bounds $[\lo^g,\hi^g]$ per group $g$
\REPEAT
  \FORALL{group $g$ and holder $h$ with a bound $[\lo^g,\hi^g]$ on $h$}
    \STATE $E_g \gets \sum_{c \in g} P_v[c][h]$
    \IF{$E_g < \lo^g$} \STATE rescale the entries of $g$ in column $h$ to sum to $\lo^g$ \ENDIF
    \IF{$E_g > \hi^g$} \STATE rescale the entries of $g$ in column $h$ to sum to $\hi^g$ \ENDIF
  \ENDFOR
  \STATE renormalize every card row to sum to 1
\UNTIL{all bounds hold within tolerance, or 4 iterations}
\end{algorithmic}
\end{algorithm}

where $G = |\text{possible}_{hj} \cap g|$ is the number of possible
group cards, $\vacant_h$ the target's unknown cards, and
$U = \sum_{h} \vacant_h$ the total unknown pool. The marginal belief is
$\belief_{hjc} = \mathbb{E}[m]/G$ for each card $c\in g$. When
constraints overlap or targets compete for the same cards, we apply
iterative proportional fitting (Sinkhorn), already a natural extension
of the column-wise normalization step: one
additional rescaling pass enforces the group bound on expected
counts (Algorithm~\ref{alg:ipf}), and 3--4 iterations suffice.

The nearly-acyclic group structure makes a single forward pass
sufficient (at most two when $\hi=0$ cascades). Cardinality rules
use the same declarative format as per-card rules, adding
\emph{group}, \emph{atmost}, and \emph{atleast} fields,
and the world filter enforces them at generation time.

\section{Rule Language}
\label{sec:rule-language}

The same declarative rule base supports multiple reasoning mechanisms
beyond belief-state updates, as illustrated in Figure~\ref{fig:framework}.

\begin{definition}[Rule]
A rule is a tuple $(n, t, \phi, C, p, \gamma, d)$ where $n$ is
a unique name, $t$ is the rule type---one of $\mathsf{init}$,
$\mathsf{dropSuit}$, $\mathsf{obeySuit}$, $\mathsf{dropTrump}$,
and $\mathsf{obeyTrump}$---$\phi$ is a Boolean condition
over the predicate vocabulary, $C$ is a card set expression to exclude
from the target, $p \in \{\mathsf{init}, \mathsf{always},
\mathsf{early}, \mathsf{late}\}$ is the phase,
$\gamma \in [0,100]$ a confidence value, and $d$ is a
natural-language description.
\end{definition}

Phases are dynamic: the engine can
promote rules between phases as different card recommenders take
control---belief-based inference in the early game yields to exact
endgame solvers when the state space collapses.

The condition language uses the predicates listed in
Table~\ref{tab:predicates}. Variables $\mathsf{icard}, \mathsf{ocard}, \mathsf{mcard}$
denote the lead card, response card, and third card respectively.

\begin{table}[t]
\centering
\small
\caption{Predicate semantics and scope. Scope: a = all card games,
  t = trick-taking.}
\label{tab:predicates}
\begin{tabular}{c|l|c}
\toprule
\textbf{Predicate} & \textbf{Semantics} & \textbf{Scope} \\
\midrule
$\mathsf{suit}(c)$ & $\sigma(c)$, suit of card $c$ & a \\ \hline
$\mathsf{value}(c)$ & point value of $c$ & t \\ \hline
$\mathsf{leads}(a,b) = a$ & $a \succ b$ in trump order & t \\ \hline
$\mathsf{card\_in}(c, S)$ & $c \in S$ (set membership) & a \\ \hline
$\mathsf{played}()$ & $|\mathsf{played}|$, cards on table & a \\ \hline
$\mathsf{score}(R)$ & points won by role $R$ & t \\ \hline
$\mathsf{issuer} = R$ & who led the trick & a \\ \hline
$\mathsf{perspective} = V$ & which observer & a \\ \hline
$\mathsf{suit\_complete}(c, N)$ & $N$ non-trumps of $\sigma(c)$ seen & t \\
\bottomrule
\end{tabular}
\end{table}

\begin{example}[Core drop-suit rule]
\[
\begin{array}{ll}
\text{Name:} & \texttt{ds\_core\_no\_more\_of\_that\_suit} \\
\text{Type:} & \mathsf{dropSuit} \\
\text{Condition:} & \mathsf{true} \\
\text{Exclusion:} & \mathsf{suit}(\mathsf{icard}) \cap \overline{\mathsf{trumpmask}} \\
\text{Confidence:} & 100\% \\
\end{array}
\]
``If a player fails to follow a non-trump lead, they hold no more cards of
that suit.'' This follows directly from the rules of trick-taking and is
common to all such games.
\end{example}

\begin{example}[Score-threshold inference]
\[
\begin{array}{ll}
\text{Name:} & \texttt{ds\_score\_defense\_aces} \\
\text{Type:} & \mathsf{dropSuit} \\
\text{Condition:} & \mathsf{issuer} = \mathsf{O2} \land
  \mathsf{score}(\mathsf{D}) > 30 \land \mathsf{score}(\mathsf{O}) < 60 \\
& \land\; \mathsf{score}(\mathsf{GS}) - \mathsf{value}(\mathsf{ocard}) \geq 49 \\
\text{Exclusion:} & \mathsf{ACES} \cap \overline{\mathsf{trumpmask}} \\
\text{Confidence:} & 100\% \\
\end{array}
\]
``The defense needs $\leq 49$ points to win. A GS player wouldn't discard an
Ace if it would push them over the threshold. Therefore the dropping player
doesn't hold the Ace.''
\end{example}

\begin{example}[Human-annotated discard heuristic]
\[
\begin{array}{ll}
\text{Name:} & \texttt{hum\_discard\_ten\_no\_ace} \\
\text{Type:} & \mathsf{dropSuit} \\
\text{Condition:} & \mathsf{ocard} \in \mathsf{TENS} \cap \overline{\mathsf{trumpmask}} \\
\text{Exclusion:} & \mathsf{ACES} \cap \mathsf{suit}(\mathsf{ocard}) \\
\text{Confidence:} & 100\% \\
\end{array}
\]
``If a defender discards a Ten on a trick they cannot win, they do not hold
the Ace of that suit---otherwise they would have played the higher-value Ace
first.'' Extracted from an expert strategy document, this rule
has confidence 100\% and was promoted to the core rule set.
\end{example}

\begin{example}[Cardinality constraint]
\[
\begin{array}{ll}
\text{Name:} & \texttt{hum\_grand\_ace\_weak\_jacks\_pairwise} \\
\text{Type:} & \mathsf{cardinality} \\
\text{Condition:} & \mathsf{game} = \mathsf{Grand}
  \land \mathsf{issuer} = \mathsf{D} \\
& \land\; \mathsf{played}(\mathsf{Jacks}) = \emptyset
  \land \mathsf{icard} \in \mathsf{ACES} \setminus \mathsf{Jacks} \\
\text{Group:} & g_{\textit{BJ}} = \{\clubsuit J, \spadesuit J\} \quad\text{(black Jacks)}\\
\text{Bound:} & 0 \le |H_{\mathsf{D}} \cap g_{\textit{BJ}}| \le 1 \\
\text{Confidence:} & 80\% \\
\end{array}
\]
``If the declarer leads a side-suit Ace in a Grand game without having
pulled any Jack, he holds at most one of the two black Jacks.''
If the declarer subsequently \emph{plays} one black Jack, the initial-hand
bound subtracts the played card, closing the remaining interval automatically.
\end{example}

The extracted rules span core obligations (drop-suit, drop-trump),
value-based refinements, score-threshold heuristics---the largest
group being 67 trump-ordering refinements---entire-trump deduction,
and initialization. For sound pruning in search under imperfect
information, every rule currently carries confidence $100\%$; lower
values enable Bayesian softening.

A code generator compiles rules into belief-update functions, automatically
handling predicate translation, card set bitmask operations, and perspective
routing to the correct $\mathcal{K}[h][v]$ entries for all three observers.
Algorithm~\ref{alg:compile} formalizes this pipeline: for each combination
of observation type and trick context, rules are grouped by phase and
target, emitted as guarded bitmask updates, and wrapped with perspective
routing that distributes accumulated exclusions to all viewer-obeyer pairs.

\begin{algorithm}[t]
\caption{Code generation: DSL rules $\to$ executable inference functions.}
\label{alg:compile}
\begin{algorithmic}[1]\footnotesize
\REQUIRE Rule base $\mathcal{R}$, each $r = (n, t, \phi, C, p, \gamma, d)$
\ENSURE Generated source file, drop-in replacement for hand-written inference
\FORALL{observation type $t$ (drop-suit, obey-suit, drop-trump, obey-trump)}
  \FORALL{trick context (incomplete, complete)}
    \STATE Emit function header with implicit trick-card parameters
    \STATE Group rules by phase: early, always, late
    \FORALL{rule $r$ with $\mathsf{target} = \mathsf{obeyer}$}
      \STATE Translate $\phi_r$
      \STATE Translate $C_r$
      \STATE Emit guarded union into shared exclusion variable $\mathsf{adding}$
    \ENDFOR
    \FORALL{rule $r$ with $\mathsf{target} \neq \mathsf{obeyer}$}
      \STATE Emit direct write to $\mathcal{K}[\mathsf{target}][v]$, bypassing $\mathsf{adding}$
    \ENDFOR
    \STATE Emit routing: $\mathcal{K}[\mathsf{obeyer}][v] \gets \mathcal{K}[\mathsf{obeyer}][v] \cup \mathsf{adding}$
  \ENDFOR
\ENDFOR
\STATE Emit specialized functions for non-standard rules
\end{algorithmic}
\end{algorithm}

\begin{figure*}[t]
\centering
\begin{minipage}{0.48\textwidth}
\centering
\includegraphics[width=\textwidth]{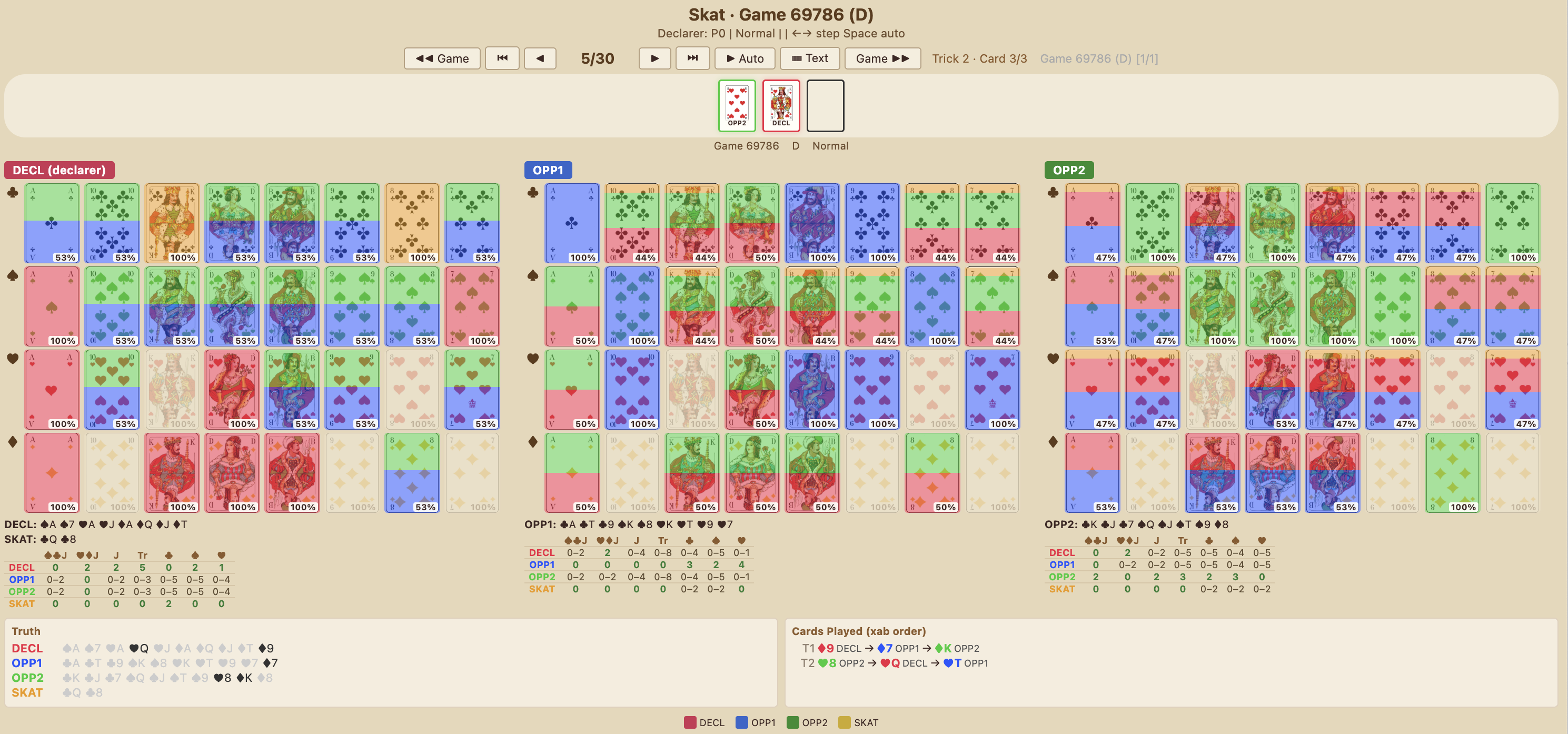}
\end{minipage}\hfill
\begin{minipage}{0.48\textwidth}
\centering
\includegraphics[width=\textwidth]{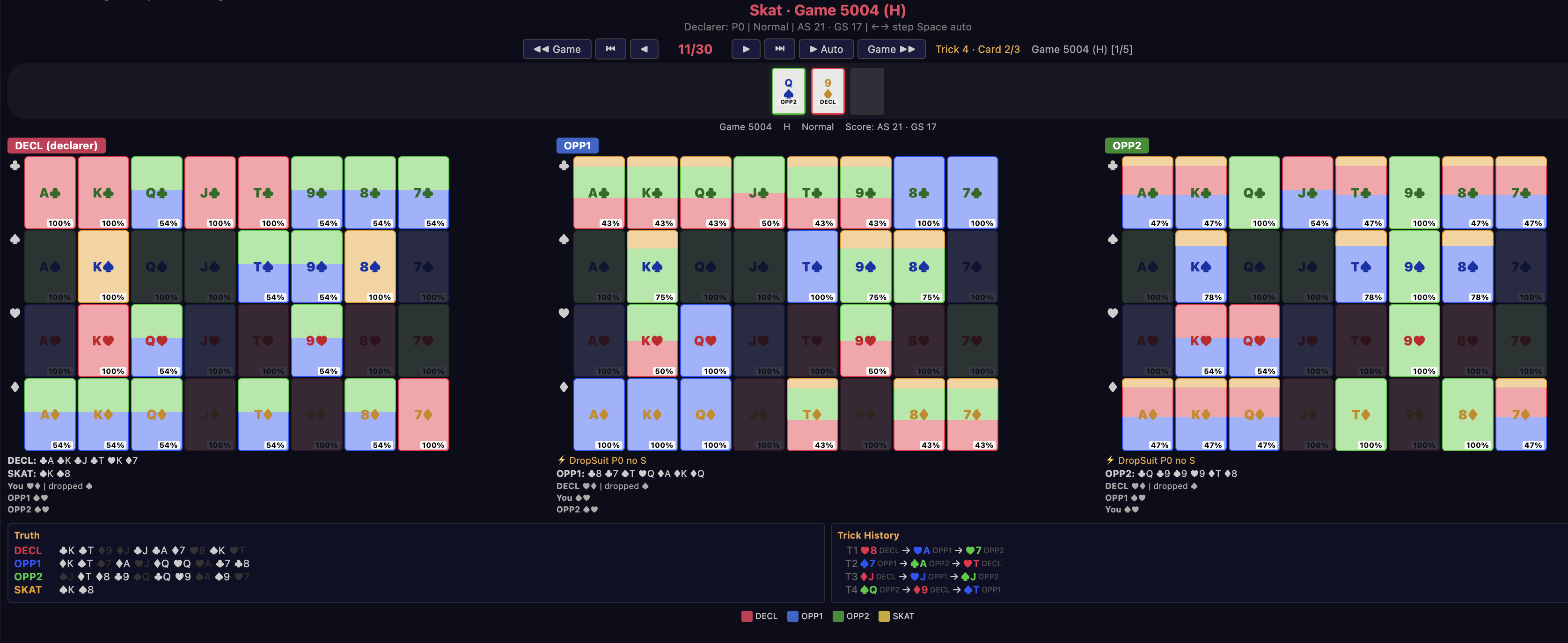}
\end{minipage}
\caption{Visualization of the belief state (two different game
snapshots). Cards are color-filled by holding probability
(red~=~declarer, blue~=~opponent~1, green~=~opponent~2, gold~=~skat).
The interval table shows group lo--hi bounds (exact in green).
Left: the cardinality group view without applied rules; right: an
annotated state with the applied rules highlighted.}
\label{fig:beliefviz}
\end{figure*}

\noindent The framework separates game-specific vocabulary from universal inference
patterns. All predicates generalize to any trick-taking game; Bridge, Belote, Hearts, and Spades differ only in
trump hierarchy and signaling conventions. The $\neg\mathsf{holds}$
constraints serve both belief-state maintenance and search---partitioning
unseen cards for the next search phase, guiding tree search with $O(1)$
follow-suit pruning, and feeding bitmask enumeration or probability matrices.

As a third layer beyond extraction and extension, we elicited knowledge
from an expert strategy document on Skat inference principles,
translating 11~additional DSL rules with LLM assistance---each manually
verified---covering trump management,
discarding conventions, and overtake heuristics. Of 10~candidate
principles, 5~translate cleanly to $\neg\mathsf{holds}$ exclusions; two
count constraints are expressible through the cardinality mechanism of
Section~4, while two distribution patterns remain deferred. The most
active rule infers that a defender who does not overtake with a trump
Ace or Ten lacks those cards.

To verify that rules fire as intended, we instrumented the code generator
to emit runtime traces: each rule logs its trigger cards, target,
excluded set, and observer perspective. A replay tool reconstructs the
game state from the original deal and annotates each trick with the rules
that fired. A side-by-side mode compares two rule configurations,
highlighting differences. This pipeline caught operator precedence bugs
and incorrect phase ordering during development. Round-trip verification
against the original engine would have been infeasible without per-trick
traceability.

\subsection{Symmetric Derivation}

Rules transfer across symmetric game positions by conjugation. Let
$\sigma$ be a \emph{structure-preserving deck automorphism}: in a
suit game, $\sigma$ permutes the three non-trump suits and fixes the
trump suit pointwise; in a Grand, $\sigma$ permutes all four suits
while preserving the Jack order. Then $\sigma$ commutes with the
suit-follow rule and with the trump order.

\begin{lemma}[Symmetric derivation]
\label{lem:symmetry}
If $\sigma$ is a structure-preserving deck automorphism and rule
$r = (n, t, \phi, C, p, \gamma, d)$ is sound, then the conjugated
rule $\sigma(r)$, obtained by replacing every card constant $c$
occurring in $\phi$ and $C$ by $\sigma(c)$, is sound, and its
exclusion set is the $\sigma$-image of the exclusion set of $r$.
\end{lemma}
\begin{proof}
Soundness of a rule means: whenever $\phi$ holds, no card of $C$ can
be with the target---a statement about the game's move semantics.
$\sigma$ relabels the deck while commuting with the move semantics,
so it maps valid inferences to valid inferences and exclusion sets to
exclusion sets. Derived rules inherit phase and confidence from their
source and are stored as first-class rules with a provenance note.
\end{proof}

\begin{algorithm}[t]
\caption{Symmetric derivation: expand the rule base by conjugation.}
\label{alg:symm}
\begin{algorithmic}[1]\footnotesize
\REQUIRE rule base $\mathcal{R}$, structure-preserving automorphism group $A$
\ENSURE extended rule base with conjugated rules
\FORALL{$r \in \mathcal{R}$}
  \FORALL{$\sigma \in A$}
    \STATE $r' \gets \sigma(r)$ \COMMENT{replace card constants in $\phi$ and $C$}
    \IF{$r'$ is new (not equivalent to an existing rule)}
      \STATE add $r'$ with the phase and confidence of $r$; note the provenance
    \ENDIF
  \ENDFOR
\ENDFOR
\end{algorithmic}
\end{algorithm}

In a suit game the automorphism group is the permutation group
$S_3$ on the three non-trump suits, so every suit-specific rule
yields its two sibling-suit variants; the engine applies the lemma to
derive, for example, the $\clubsuit$-void rules from the
$\heartsuit$-void rules and the red-Jack bounds from the black-Jack
bounds, which yields the 36 extended rules of
Section~\ref{sec:experiments}.

\section{Experimental Results}
\label{sec:experiments}

We use a state-of-the-art Skat solver as the baseline
\cite{edelkamp2019socs,edelkamp2020ecai,edelkamp2020dynamic,edelkamp2021paranoia,edelkamp2022hope,edelkamp2024ki,edelkamp2025outer,armbruester2026quantum}
engine. Experiments run on an Apple M3 Pro with 36\,GB RAM, compiled with
GCC~14 \texttt{-O3}. All games are played to completion with full
belief-state updates.

The 122~extracted rules regenerate belief-update code
whose game outcomes match the original on 99~of~100~games and on
965~of~1{,}000~games (96.5\%). Auditing attributes the mismatches to
three bugs in the original code---including incorrect player-view
assignments and missing game-state guards---rather than to the extracted
rules: the declarative engine reproduces the original's \emph{intended}
behavior and corrects its errors. In the single differing game of the
100-game sample (game~81, $-7$~points), the corrected view assignment
produces stronger opponent inference than the buggy original---a net
improvement in belief accuracy.

The 36~extended rules cover edge cases not reached in
smaller samples. The 11~human-annotated rules from an expert strategy
document produce no measurable harm in self-play
(Table~\ref{tab:selfplay}); additional rule extensions were evaluated
with the same outcome. As Proposition~\ref{prop:monotone} predicts,
every layer only tightens beliefs---the observed no-harm behavior is
structural, not incidental.

We evaluate play strength through self-play on
83{,}844 Skat games, distributed across 16~parallel processes on an Apple
M3~Pro, with each experiment completing in under one hour. Each game records which of the three positions
(declarer, middle hand, rear hand), yielding an 8-class
encoding: 111 means the AI wins all three positions, 000 means it loses all three.
Table~\ref{tab:selfplay} reports
the distribution for the original engine, the 122-rule engine, the engine with
11 additional human-annotated rules, and the full configuration
adding 6~cardinality constraint rules (139~rules total).

\begin{table}[t]
\centering
\caption{Self-play on 83{,}844 games. Win\% = AI wins all three positions
(111). Score = average Seeger points per game (higher is better).}
\label{tab:selfplay}
\resizebox{\columnwidth}{!}{%
\begin{tabular}{lrrrr}
\toprule
Outcome & Original & 122 rules & +Human & +Card. \\
\midrule
000 (all lose) & 5{,}008 & 4{,}913 & 4{,}893 & 4{,}708 \\
001            & 3{,}294 & 3{,}389 & 3{,}409 & 3{,}594 \\
010            &   370  &   361  &   359  &   338  \\
011            & 4{,}720 & 4{,}729 & 4{,}731 & 4{,}752 \\
100            & 6{,}083 & 5{,}890 & 5{,}886 & 5{,}593 \\
101            & 7{,}691 & 7{,}884 & 7{,}888 & 8{,}181 \\
110            & 1{,}195 & 1{,}127 & 1{,}125 & 1{,}035 \\
111 (all win)  & 54{,}081 & 54{,}149 & 54{,}151 & 54{,}241 \\
\midrule
Win\%          & 84.65 & 85.09 & 85.13 & 85.84 \\
Score          & 946.3 & 952.9 & 953.6 & 962.6 \\
\bottomrule
\end{tabular}%
}
\end{table}

The 122-rule engine improves winning percentage by $+0.44$ and average
score by $+6.6$ over the original. Adding human-annotated rules yields a
further $+0.04$ Win\% and $+0.6$ score. The cardinality configuration
adds a further $+0.71$ Win\% and $+9.0$ score, with 111 (all-win) rising
to 54{,}241.

On $83{,}844$ games the standard error of a win-rate difference is
$0.17$ percentage points, so the extracted layer ($+0.44 \pm 0.34$
pp, 95\% CI) and the cardinality layer ($+0.71 \pm 0.34$ pp) are
highly significant, while the human-annotated layer ($+0.04 \pm
0.34$ pp) is within noise on this sample---exactly the refinement
chain of Proposition~\ref{prop:monotone}: each layer is sound, and
the measurable gains come from the two data-backed layers.

Absolute win-rate gains look small because the baseline is strong and
the game is noisy: with 15 percentage points of headroom to perfect
play, $+0.71$~pp closes $4.7\%$ of the remaining error, and Skat's
deal randomness---with strong cards even weak players win---compresses
the margin available to improved inference. This is exactly why
$83{,}844$ games are needed to separate the layers statistically; at
that scale, every significant win-rate increase corresponds to
strictly superior play.

Six cardinality constraint rules were added to the
rule set: the Grand Ace-lead rules (at most two Jacks, at most one black
Jack), the declarer Ten-lead rule (at least one supporting card in the
led side suit), the skat-Ten rule (declarer drops a side suit whose Ten is
unseen $\rightarrow$ Ten in the skat), the suit-game Ace-lead rule (at
most five trump), and the first-trump-pull rule (a low opening trump
card rules out three Jacks). On a 1{,}000-game benchmark, the drop-Ten
rule fires in 612~games (most frequent), the Grand Ace-lead rule in
56~games, the Ten-support rule in 9~games, and the suit-game trump bounds
in the remaining Grand/suit deals. In a controlled A/B~test on 100~games (one
aborted), 23~of~99~played games show different play sequences.

The world filter is installed in all world-generation entry points, covering both the declarer's and
opponents' perspectives.
Bounds are maintained by a double-pass update: when $\hi=0$ triggers,
the belief entries for the affected group are zeroed and a second pass
redistributes probability mass. Custom group
masks (e.g., black Jacks) enable domain-specific constraints without
engine changes. The follow-on inference (playing one black
Jack~$\rightarrow$~remaining bound~$=0$) is automatically enforced by
subtracting already-played cards from the initial-hand bound. Figure~\ref{fig:beliefviz} shows the
resulting group-interval display.

Rule evaluation adds negligible overhead. On a 1{,}000-game
benchmark, the original processes 393~ms per game; the rule engine
396~ms per game---a factor of just $1.01{\times}$. Each rule compiles to a few bitwise
operations on a 32-bit mask. The cardinality constraint machinery and
world filter add no measurable overhead beyond the per-card rules
themselves.

\section{Reproducibility, Deployment, and Limitations}
\label{sec:repro}

\textbf{Reproducibility.} The supplementary material bundles the
declarative rule base in the DSL of Section~\ref{sec:rule-language}---the
122 extracted rules, the symmetric derivations, the extended and the
human-annotated rules, and the code generator---together with the
reported datasets as CSV (the self-play outcome distributions, the
round-trip verification, the rule-fire counts, and the A/B test) and
a rule-annotated replay in which every trick step lists the rules
that fired. All experiments ran on an Apple M3 Pro with 36\,GB RAM,
compiled with GCC~14 \texttt{-O3}, with the self-play campaigns
distributed over 16 parallel processes.

\textbf{Deployment.} The engine, its configuration, and the
interactive belief viewer are available in the supplementary material,
where the annotated replay is self-contained and does not require a
server. During development, live deployments were reachable at
\url{http://www.skat-ki.de:9080/docs/} (documentation),
\url{http://www.skat-ki.de:9080/} (interactive play and self-play
between the engines), and \url{http://www.skat-ki.de:9080/replay}
(series sheets and per-game belief viewers); these addresses may not
remain online.

\textbf{Limitations.} The rule base is currently Skat-specific:
porting the method to another game repeats the extraction and
annotation steps, though the DSL, the generator, and the cardinality
machinery carry over unchanged. The one-sweep guarantee covers
laminar group families; arbitrary overlapping groups fall back to the
iterative proportional fitting of Section~\ref{sec:knowledge}, which
converges in practice (3--4 iterations) but lacks the same a-priori
bound. All rules currently carry confidence $100\%$, so the
probabilistic softening path remains future work, as does end-to-end
learning of new rules from play data.

\section{Conclusion and Future Work}

We presented a declarative framework for belief-state inference in
trick-taking card games, centered on the $\neg\mathsf{holds}$ primitive and
a domain-specific language of predicates. A code generator translates 122
human-readable rules into efficient C++ as a drop-in replacement for the
hand-written inference logic in a production Skat engine, verified to match
the original on 965~of~1{,}000~games, with every difference traced to
audited bugs in the original. Two additional rule layers were explored:
36~extended rules and 11~human-annotated rules from an expert strategy document.
A debugging pipeline combining instrumented code generation with trick-by-trick
replay enables per-rule auditing of inference correctness.

Self-play evaluation on 83{,}844 games shows the original engine at
84.65\% win (946.3 score), the 122-rule engine at 85.09\% (952.9),
and the full 139-rule configuration with cardinality constraints at
85.84\% (962.6), with the extracted rules, the human-annotated
rules, and the cardinality constraints each contributing to the
effect.

Extending the knowledge representation with cardinality constraints on
named card groups captures inferences inexpressible as pure card
exclusion, with the nearly-acyclic group structure making a single
forward bounds-consistency pass sufficient. Custom groups and the
follow-on inference (playing one black Jack~$\rightarrow$~remaining
bound~$=0$) illustrate the expressiveness of the formalism.

The three-layer architecture generalizes to Bridge, Belote, and other
trick-taking games, and the confidence parameter provides a foundation
for Bayesian extensions and machine-discovered rules.

Two neural integration paths are under development. Path~A (NN
strengthens hard constraints): an LSTM model trained on millions of
human and computer games refines belief probabilities and writes
high-confidence outputs back to the $\neg\mathsf{holds}$ matrix,
tightening the hard constraints used by the game solver. Path~B
(hard constraints strengthen NN inference): the NN's probability
distribution drives weighted world sampling, with hard constraints
as a logical floor, but this path currently exists only as a design. Both are
compatible with the existing engine through a lightweight serialization
layer and represent the next step in combining declarative knowledge
with learned inference.

\section*{Acknowledgments}
We thank Rainer G\"o\ss{}l for seeding the work on Inner-Learning,
Hagen Raatsch for sharing his work on neural-network belief
representation, and Marcel Petr\'a\v{n} for reading and feedback on
partial-information games. The contribution was supported by GACR project 24-12046S.

\bibliographystyle{ACM-Reference-Format}
\bibliography{references}

\end{document}